\documentclass[manuscript]{acmart}
\pdfoutput=1
\AtBeginDocument{%
  }

\setcopyright{acmcopyright}
\copyrightyear{2023}
\acmYear{2023}
\acmDOI{XXXXXXX.XXXXXXX}

\acmConference[Conference acronym 'XX]{Make sure to enter the correct
  conference title from your rights confirmation emai}{June 03--05,
  2023}{Woodstock, NY}
\acmPrice{15.00}
\acmISBN{978-1-4503-XXXX-X/18/06}

\usepackage{tikz}
\usepackage{bm}
\usepackage{hyperref}
\hypersetup{hidelinks,
	colorlinks=true,
	allcolors=black,
	pdfstartview=Fit,
	breaklinks=true}
\usetikzlibrary{positioning}
\usepackage{svg}
\usepackage{multirow}
\usepackage{array}
\usepackage{subfigure}
\newtheorem{corallary}{Corollary}%
\begin{document}

\title{Context-aware feature attribution through argumentation}

\author{Jinfeng ZHONG}
\affiliation{%
  \institution{Paris-Dauphine University, PSL Research University, CNRS UMR 7243, LAMSADE}
  \city{Paris}
  \country{France}
  \postcode{75016}}
\email{jinfeng.zhong@dauphine.eu}

\author{Elsa NEGRE}
\affiliation{%
  \institution{Paris-Dauphine University, PSL Research University, CNRS UMR 7243, LAMSADE}
  \city{Paris}
  \country{France}
  \postcode{75016}}
\email{elsa.negre@lamsade.dauphine.fr}

\renewcommand{\shortauthors}{Zhong and Negre}

\begin{abstract}
Feature attribution is a fundamental task in both machine learning and data analysis, which involves determining the contribution of individual features or variables to a model's output. This process helps identify the most important features for predicting an outcome. The history of feature attribution methods can be traced back to General Additive Models (GAMs), which extend linear regression models by incorporating non-linear relationships between dependent and independent variables. In recent years, gradient-based methods and surrogate models have been applied to unravel complex Artificial Intelligence (AI) systems, but these methods have limitations. GAMs tend to achieve lower accuracy, gradient-based methods can be difficult to interpret, and surrogate models often suffer from stability and fidelity issues. Furthermore, most existing methods do not consider users' contexts, which can significantly influence their preferences. To address these limitations and advance the current state-of-the-art, we define a novel feature attribution framework called \textbf{C}ontext-\textbf{A}ware \textbf{F}eature \textbf{A}ttribution \textbf{T}hrough \textbf{A}rgumentation (CA-FATA). Our framework harnesses the power of argumentation by treating each feature as an argument that can either support, attack or neutralize a prediction. Additionally, CA-FATA formulates feature attribution as an argumentation procedure, and each computation has explicit semantics, which makes it inherently interpretable. CA-FATA also easily integrates side information, such as users' contexts, resulting in more accurate predictions. Our experiments on two real-world datasets demonstrate that CA-FATA, or one of its variants, outperforms existing argumentation-based methods and achieves competitive performance compared to existing context-free and context-aware methods. CA-FATA also ensures the explainability of recommendations, making it a promising framework for practical applications.


\end{abstract}



\maketitle

\section{Introduction}
 Feature attribution has been a long-standing practice {in the field of machine learning} to determine the contribution of individual features or variables to a model's overall output. This method has also been applied in recommender models to explain their behaviors \cite{afchar2020making, zhong20223, rago2018argumentation,nobrega2019towards,lundberg2017unified,ribeiro2016should,zhong2022shap}. The process of feature attribution helps identify the most important features for predicting an outcome and areas for model improvement. The origin of feature attribution methods can be traced back to general additive models (GAMs) \cite{hastie2017generalized}. Although GAMs are inherently interpretable, they often suffer from limited expressivity \cite{molnar2020interpretable}. In recent years, gradient-based methods \cite{li2021deep,selvaraju2017grad,ancona2019gradient} have been employed to disentangle complex Artificial Intelligence (AI) systems. These methods determine the importance of a feature $x$ in a function $f$ by calculating the derivative of $f$ with respect to the feature $x$. However, gradient-based methods may struggle with simple tasks that require understanding a moderately local region \cite{bilodeau2022impossibility}, and interpreting such gradients can be challenging for non-experts. To address the issue of gradient-based methods, surrogate models such as LIME \cite{ribeiro2016should} and SHAP \cite{lundberg2017unified} have emerged as two prominent post-hoc explanation methods. However, the limitations of these methods have been recognized. LIME inherently suffers from stability issues \cite{visani2022statistical}, and the attribution of feature importance through mathematically formalizable properties (i.e., local accuracy, missingness, and consistency \cite{lundberg2017unified}) may not always align with users' expectations for explanations \cite{kumar2020problems}. 

In recent years, argumentation-based methods have gained significant attention in the field of eXplainable Artificial Intelligence (XAI) \cite{vassiliades2021argumentation, vcyras2021argumentative, zeng2018context}. This is due to the clear and understandable means of representing relations, such as support and attack, offered by Argumentation Frameworks (AFs), which provide explicit meanings to the computation. Under AFs, decision-making processes can be visually depicted, and optimal decisions can be explained using well-defined properties \cite{vassiliades2021argumentation}. Weighted arguments are used to represent the strength of arguments and the dialectical relations between them, such as support and attack. The strength function of arguments can be carefully designed to satisfy the generalized concepts of \emph{weak balance} \cite{rago2018argumentation} and \emph{weak monotonicity} \cite{rago2021argumentative}, which characterize how arguments influence the decision-making process (we will revisit the two notions in Section~\ref{sec:preliminaries}). These methods can be used to explain decisions made through a graphical representation of the decision-making process. Context-Aware Recommender System (CARS) \cite{adomavicius2011context} is an important research topic in recommender systems. CARS can model users' preferences under different contextual situations with finer granularity and generate more personalized recommendations adapted to users' contexts. We believe that context is also crucial in argumentation frameworks, as certain arguments that are considered ``good'' in one context may become less accurate in another context. Therefore, {it is important to leverage contexts when applying argumentation \cite{teze2018argumentative,garcia2004defeasible}}.

In light of the interpretability challenges associated with traditional feature attribution methods, it is reasonable to explore new avenues for improving the explainability of machine learning models. One such approach is to leverage argumentation techniques to attribute feature importance. In this paper, we introduce a novel framework for feature attribution called \textbf{C}ontext-\textbf{A}ware \textbf{F}eature \textbf{A}ttribution \textbf{T}hrough \textbf{A}rgumentation (CA-FATA). The framework employs argumentation to attribute importance to each feature, considering them as arguments that can either support, attack or neutralize a prediction. CA-FATA formulates feature attribution under argumentation frameworks, providing each computation with explicit semantics, thereby ensuring interpretability. Additionally, the framework allows for the integration of side information, such as user contexts, resulting in more accurate predictions. Our experiments on two real-world datasets demonstrate that CA-FATA outperforms existing argumentation-based methods and achieves competitive performance compared to existing context-free and context-aware methods. CA-FATA also ensures the explainability of recommendations, making it a promising framework for practical applications. 

\section{Related work}\label{sec:related work}

\subsection{Feature attribution}
There are numerous methods available to identify the most important features {in the field of machine learning}. Chronologically, feature attribution can be traced back to GAMs \cite{hastie2017generalized}. Recently, gradient-based methods \cite{li2021deep,selvaraju2017grad,smilkov2017smoothgrad} have gained popularity due to the widespread use of deep neural networks. Developing surrogate models \cite{lundberg2017unified, ribeiro2016should} to approximate black-box models by human-interpretable models is another approach that has been widely adopted.

\textbf{GAMs} extend linear regression models by incorporating non-linear relationships between dependent and independent variables and are typically in the following form: 

\begin{equation}\label{equa:gams}
    g(y) = f_1(x_1) + f_2(x_2) + f_2(x_2) + \dots + f_i(x_i) + \dots + f_n(x_n) 
\end{equation}
where $x_1,x_2,\dots,x_i, \dots, x_n$ are the features applied to make predictions, $g$ is called the link function and $f_i$ is the shape function \cite{lou2012intelligible}. GAMs have a built-in interpretability feature that allows tracing the contribution of each feature, thereby facilitating the identification of the most important features. Nevertheless, the expressivity of GAMs is limited compared {with more complex models such as ensemble models \cite{molnar2020interpretable}.}

\textbf{Gradient-based methods} have become a popular means of identifying important features in predictive functions $f(x_1, x_2 \dots x_i \dots x_n)$ by computing the derivative of $f$ with respect to the feature $x_i$: $\frac{\partial f}{\partial x_i}$ \cite{li2021deep,selvaraju2017grad,ancona2019gradient}. This is based on the intuition that the sensitivity of the model output to changes in the input, as quantified by $\frac{\partial f}{\partial x_i}$, reveals the importance of the feature. However, as noted by Ancona et al. \cite{ancona2019gradient}, gradient-based methods can be sensitive to noisy gradients, and may struggle with simple tasks that require understanding a moderately local region \cite{bilodeau2022impossibility}. Additionally, interpreting these gradients can be challenging for non-experts.

\textbf{Surrogate models} aim to approximate the behavior of black-box models with more interpretable ones, without revealing the internal workings of black-box models. Among the many surrogate models available, LIME \cite{ribeiro2016should} and SHAP \cite{lundberg2017unified} have emerged as two prominent methods. LIME and SHAP can both be regarded as additive feature attribution methods and can be unified as shown in Equation~\ref{equa:gams} \cite{lundberg2017unified}. However, since LIME and SHAP are post-hoc methods, it has been argued that the explanations may not accurately reflect the reasoning of black-box models, and may even be misleading \cite{rudin2019stop}. Additionally, studies have demonstrated that LIME and SHAP can be easily deceived into producing innocuous explanations that do not reveal underlying biases \cite{slack2020fooling}.

{In view of the limits of existing feature attribution methods,} we define a novel feature attribution framework that leverages argumentation. Specifically, the framework maps the prediction process into argumentation procedures, treating features as arguments and incorporating a strength function that reflects how each argument (feature) influences the prediction.
 
\subsection{Argumentation frameworks}
 Among the existing argumentation frameworks, three types can be identified: Abstract Argumentation Framework (AAF) \cite{dung1995acceptability}, Bipolar Argumentation Framework (BAF) \cite{cayrol2005acceptability}, Tripolar Argumentation Framework (TAF) \cite{gabbay2016logical}. An AAF is composed of a set of pairs $<\mathcal{A},\mathcal{R}^->$, where $\mathcal{R}^-$ denotes a set of attack relations between arguments such that $ {\forall}\mathbf{a}_1, \mathbf{a}_2 \in \mathcal{A}$, $(\mathbf{a}_1, \mathbf{a}_2) \in \mathcal{R}^-$ denotes that argument $\mathbf{a}_1$ attacks argument $\mathbf{a}_2$. The relation ``attacks'' indicates a contradiction between two arguments. For example, considering $\mathbf{a}_1$ ``This user does not like the feature of this item (one actor of a movie)'' and $\mathbf{a}_2$ ``This item can be recommended to this user''.  It is evident that $\mathbf{a}_1$ attacks $\mathbf{a}_2$. BAFs contain a set of triplets, $<\mathcal{A},\mathcal{R}^-, \mathcal{R}^+>$, $\mathcal{R}^-$ represents attack as in AAF. Similarly, $\mathcal{R}^+$ denotes a set of support relations between arguments such that $ {\forall}\mathbf{a}_1, \mathbf{a}_2 \in \mathcal{A}$, $(\mathbf{a}_1, \mathbf{a}_2) \in \mathcal{R}^+$ denotes that argument $\mathbf{a}_1$ supports argument $\mathbf{a}_2$. TAFs contain a set of quadruples: $<\mathcal{A},\mathcal{R}^-, \mathcal{R}^+, \mathcal{R}^0>$, where $\mathcal{R}^-$ represents the attack relations, $\mathcal{R}^+$ denotes the support relations and $\mathcal{R}^0$ means neutralizing relations. In this work, we have chosen to adopt TAFs that comprise three types of relations between arguments: attack, support, and neutralizing. {This is because features of items may support, attack, or neutralize the recommendation of items, indicating users' preferences towards features.} As we will discuss later, the strength of the arguments in the TAFs presented in this paper is based on the users' ratings towards features. Consequently, the TAFs in this paper can be viewed as instances of Weighted Argumentation Frameworks as defined in \cite{bistarelli2022labelling} where the weight denotes the strength of arguments.
Argumentation-based methods have gained attention for building decision-support tools due to their natural explainability. In the field of recommender systems, such methods have received increasing attention. For example, Briguez et al. \citep{briguez2012towards,briguez2014argument} proposed a music recommendation system based on Defeasible Logic Programming (DeLP) rules \citep{garcia2004defeasible}, which can handle incomplete and contradictory information in RS. Toniolo et al. \citep{toniolo2020dialogue} proposed a decision-making system that provides medical treatments to patients with multiple chronic health problems. To justify the recommended treatments during argumentation dialogues, they employed \emph{Satisfiability Modulo Theories} \cite{barrett2018satisfiability}.

Our research is closely related to two previous works: the \emph{Aspect-Item framework (A-I)} introduced in \cite{rago2018argumentation,rago2021argumentative} and the \emph{Attribute-Aware Argumentative Recommender ($A^3R$)} proposed in \cite{zhong20223}. Both A-I and $A^3R$ use argumentation to predict users' ratings towards items, treating items and features as arguments that may attack or support each other to explain recommendations in an argumentative manner. However, these methods do not consider the influence of user contexts.

Several attempts have been made to integrate CARSs with argumentation-based methods. For example, in \cite{teze2018argumentative}, the authors propose to generate context-aware recommendations using DeLP \cite{garcia2004defeasible}, where contexts are presented as a set of conditions to derive concluding expressions. However, only a prototype in a mobile robotic environment \cite{luis2008decision} was described to show the system's workings. Similarly, \cite{zeng2018context} proposes a context-based decision-making framework with argumentation, where DeLP is also applied. The authors model contexts as rules in an assumption-based argumentation framework (an extension of AF). In this framework, contexts moderate the reachability of goals from decisions and therefore influence the decision-making process. The authors further illustrate the proposed framework through an example of choosing the appropriate treatment for a patient threatened by blood clotting. In previous research, authors such as Teze et al. \cite{teze2018argumentative} and Zeng et al. \cite{zeng2018context} have proposed decision-making frameworks that utilize argumentation theory to adapt to contextual factors. However, these approaches have not addressed the problem of rating prediction, which is a classic task in recommender systems. Other works, such as Budan et al. \cite{budan2020proximity}, have considered the role of context in argumentation by developing methods to compute the similarity of arguments in different contexts. {Contrary to \cite{budan2020proximity}}, our paper focuses on examining how contextual factors can influence decision-making processes in an argumentative way.

\subsection{Context-aware recommender systems (CARSs)}Context is a notion that researchers from various domains have extensively studied, but a consensus on its definition has not been reached. In this paper, we adopt the definition proposed by Dey et al. \cite{abowd1999towards}: ``Context is any information that can be used to characterize the situation of an entity.'' A contextual situation, denoted as $cs$, is typically composed of several contextual conditions. For instance, the situation of a user can be characterized by various contextual factors such as ``Companion'', ``Day of the week'', and ``Location''. Each contextual factor consists of several possible values, which are referred to as contextual conditions. For instance, the contextual factor "Companion" may include possible values such as ``Companion'' include: ``With family'', ``With friends'', and so on.  Therefore, a contextual situation is a combination of different contextual conditions. To illustrate, a user's situation could be defined as $(Sunday, at\_home, with\_family)$. In this paper, we consider \emph{context} and \emph{contextual situation} to be equivalent terms. CARSs use contextual information to model users' preferences with greater precision, allowing them to provide more personalized recommendations that are tailored to their contextual situations. Various context-aware models have been proposed, including CAMF-C \cite{baltrunas2011matrix} and FM \cite{rendle2010factorization}, which are factorization-based methods. Recently, neural networks have also been applied to capture users' preferences under different contextual situations \cite{unger2020context}. However, explaining recommendations generated by these methods is challenging: {the semantics of the learned latent factors are usually unclear,} which is why we turn to argumentation as an approach that is both intuitive and explainable.
 \section{Preliminaries and Problem formulation}\label{sec:preliminaries}

 \begin{table}[t]
\centering
\caption{The key notions in this paper}
\begin{tabular}{@{}ll@{}}
\toprule
Notion                    & Description                                  \\ \midrule
$cf$ & A contextual factor                  \\
$C$&The contextual factors to characterize the situation of users\\
$cd$&A contextual condition\\
$\pi_u^{cf}$&The importance of contextual factor $cf$ to user $u$\\
$cs = (cd_1, cd_2, cd_3, \dots)$& A contextual situation\\
$\pi_u^t$&  The importance of type $t$ to user $u$  \\
$t_i$ & The feature types of item $i$
\\
$at_i$&The features of item $i$\\
 $at_i^{t}$& The features of item $i$ of type $t$ (e.g. A movie may have several genres)\\
 $\mathcal{P}_u^{at}$ & The predicted rating of user $u$ towards feature $at$ \\
 $\mathcal{R^+}, \mathcal{R}^-, \mathcal{R}^0$ & Support, attack and neutral relations among arguments                  \\
$\mathcal{R}^-$(a,b) & Argument $a$ attacks argument $b$, the same for ``+'' and ``0''  \\
$\mathcal{A}$             & A set of arguments                           \\
$rec^i$ &An argument stating ``the item can be recommended to the target user''\\
$\mathcal{R}^+(rec^i) = \{at|(at,i) \in \mathcal{R}^+\}$&The arguments (features) that support $rec^i$\\
$\mathcal{R}^-(rec^i) = \{at|(at,i) \in \mathcal{R}^+\}$&The arguments (features) that attack $rec^i$\\
$\mathcal{R}^0(rec^i) = \{at|(at,i) \in \mathcal{R}^+\}$&The arguments (features) that neutralize $rec^i$\\
$<\mathcal{A}, \mathcal{R}^{-}, \mathcal{R}^{+}, \mathcal{R}^{0}>$& A tripolar argumentation framework\\
$\sigma(a)$               & The strength of argument $a$                 \\
 \bottomrule
\end{tabular}
\label{tab:notions}
\end{table}

To standardize the terminology used in the rest of this paper, we present the key notions in Table~\ref{tab:notions}. Note that the bold font is used to represent the vectors that denote $cd$, $cs$, $cf$, $u$, $t$, and $at$.

\subsection{Weak balance and weak monotonicity}
The concept of \emph{weak balance}, as defined in \citep{rago2018argumentation}, is a generalization of the notion of ``strict balance'' proposed in \citep{baroni2018many}. Similarly, the idea of \emph{weak monotonicity}, defined in \citep{baroni2019fine}, is a generalization of the concept of ``strict monotonicity'' as defined in \citep{rago2021argumentative}. The two properties allow for deriving intuitive explanations in an argumentative way. Essentially, the concept of \emph{Weak balance} concerns the impact of an argument on its affectees when the argument is the sole factor affecting them, while the idea of \emph{Weak monotonicity} focuses on how the potency of an argument changes when one of its affecters is silenced relative to the neutral point.

\textbf{Weak balance:} One potential approach to analyze the effects of one argument on another is to isolate the affecter and examine its impact on the affectee. The intuition behind this approach is that if the affecter increases the strength of the affectee, then it supports the affectee. This idea has been formalized as \emph{weak balance} in \citep{rago2018argumentation}. According to \emph{weak balance}, relations under argumentation frameworks such as attacks (or supports, neutralizes) can be characterized as connections among affecters and affectees in the following way: if one affecter is isolated as the single argument that affects the affectee, then the former reduces (or increases, does not change) the latter's predicted rating with respect to the neutral point. In other words, relations under argumentation frameworks can be analyzed by examining the connections between affecters and affectees. Formally, \emph{weak balance} can be defined as follows:
\begin{definition} 
Given a TAF $<\mathcal{A}, \mathcal{R}^{-}, \mathcal{R}^{+}, \mathcal{R}^{0}>$, $\sigma(at)$ satisfies the property of \emph{weak balance} if, for any $a,b \in \mathcal{A}:$
\begin{center}
$\bullet$ if  $\mathcal{R}^+(rec^i) = \{at\}$, $\mathcal{R}^-(rec^i) = \emptyset$ and $\mathcal{R}^0(rec^i) = \emptyset$ then $\sigma(rec^i) > 0$;
\\
$\bullet$ if  $\mathcal{R}^-(rec^i) = \{at\}$, $\mathcal{R}^+(rec^i) = \emptyset$ and $\mathcal{R}^0(rec^i) = \emptyset$ then $\sigma(rec^i) < 0$;
\\
$\bullet$ if  $\mathcal{R}^0(rec^i) = \{at\}$, $\mathcal{R}^+(rec^i) = \emptyset$ and $\mathcal{R}^-(rec^i) = \emptyset$ then $\sigma(rec^i) = 0$.
\end{center}
\end{definition}

\textbf{Weak monotonicity :} To study the effects of one argument on another, another alternative approach is to mute the affecter, setting its strength to 0, and observe how the strength of the affectee changes. This idea is intuitive: if the affecter supports the affectee, then muting the affecter would decrease the strength of the affectee; if the affecter attacks the affectee, then muting the affecter would increase the strength of the affectee; if the affecter neutralizes the affectee, then muting the affecter would not change the strength of the affectee. This intuition has been formalized as \emph{weak monotonicity} in \citep{baroni2019fine}. This property is formulated for two TAFs: from $<\mathcal{A}, \mathcal{R}^{-}, \mathcal{R}^{+}, \mathcal{R}^{0}>$ to $<\mathcal{A}^{\prime}, \mathcal{R}^{{-}^{\prime}}, \mathcal{R}^{{+}^{\prime}}, \mathcal{R}^{{0}^{\prime}}>$ after modifying certain arguments (e.g. muting certain features). Formally, \emph{weak monotonicity} is defined as follows:

\begin{definition} 
Given two TAFs $<\mathcal{A}, \mathcal{R}^{-}, \mathcal{R}^{+}, \mathcal{R}^{0}>$ and $<\mathcal{A}^{\prime}, \mathcal{R}^{{-}^{\prime}}, \mathcal{R}^{{+}^{\prime}}, \mathcal{R}^{{0}^{\prime}}>$, $(x,y) \in (\mathcal{R}^{-} \cup \mathcal{R}^{+} \cup \mathcal{R}^{0}) \cap (\mathcal{R}^{{-}^{\prime}}\cup\mathcal{R}^{{+}^{\prime}}\cup\mathcal{R}^{{0}^{\prime}})$. $\sigma$ satisfies \emph{weak monotonicity} at $(x,y)$ if, as long as $\sigma(x) = 0$ in $<\mathcal{A}^{\prime}, \mathcal{R}^{{-}^{\prime}}, \mathcal{R}^{{+}^{\prime}}, \mathcal{R}^{{0}^{\prime}}>$ and $\forall z \in [(\mathcal{R}^{-}(y) \cup \mathcal{R}^{+}(y)  \cup \mathcal{R}^{0})(y)  \cap (\mathcal{R}^{{-}^{\prime}}(y) \cup\mathcal{R}^{{+}^{\prime}}(y) \cup\mathcal{R}^{{0}^{\prime}})(y)]\backslash\{x\}$ and for $\sigma(z) = s$ in  $<\mathcal{A}, \mathcal{R}^{-}, \mathcal{R}^{+}, \mathcal{R}^{0}>$ and $\sigma(z) = s^{\prime}$ in $<\mathcal{A}^{\prime}, \mathcal{R}^{{-}^{\prime}}, \mathcal{R}^{{+}^{\prime}}, \mathcal{R}^{{0}^{\prime}}>$, $s = s^{\prime}$. Then the following holds for $\sigma(y) = v$ in  $<\mathcal{A}, \mathcal{R}^{-}, \mathcal{R}^{+}, \mathcal{R}^{0}>$ and $\sigma(y) = v^{\prime}$ in $<\mathcal{A}^{\prime}, \mathcal{R}^{{-}^{\prime}}, \mathcal{R}^{{+}^{\prime}}, \mathcal{R}^{{0}^{\prime}}>$
\begin{center}
$\bullet$ if $x \in \mathcal{R}^{-}(y) \cap \mathcal{R}^{{-}^{\prime}}(y)$, then $v^{\prime} > v $;
\\
$\bullet$ if $x \in \mathcal{R}^{+}(y) \cap \mathcal{R}^{{+}^{\prime}}(y)$, then $v^{\prime} < v $;
\\
$\bullet$ if $x \in \mathcal{R}^{0}(y) \cap \mathcal{R}^{{0}^{\prime}}(y)$, then $v^{\prime} = v $.
\end{center}
\end{definition}
\begin{figure}[t]
    \centering
      \begin{tikzpicture}
\node[circle,draw,text = green] (a) at (7,1) {at1};
\node[circle,draw,text = red] (b) at (10,-2) {at2};
\node[circle,draw] (c) at (13,1) {at3};
\node[circle,draw,text=black,fill=lightgray] (d) at (10,0) {item};
\draw[green,->,thick] (a) to [in = 120, out  = -120] (d);
\node[text = green] at (8.5,0.90) {+0.52};
\draw[red,->,thick] (b) -- (d);
\node[text = red] at (10.4,-1) {-0.11};
\node at (11.5,0.7) {0};
\draw[->,thick] (c) to [in = 90, out  = -120] (d);
\end{tikzpicture} 
   \caption{A graphical representation of an argumentation procedure in a recommendation scenario. Each node represents an argument, The value on the arc denotes the strength and polarity of the argument, ``+'' denotes supports, ``-'' denotes attacks, and ``0'' denotes neutralizes.}
    \label{fig:explanation}
\end{figure}
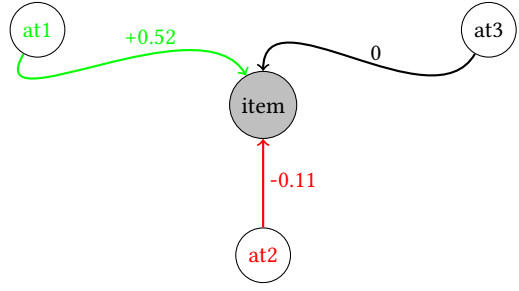

\subsection{Problem formulation}
 Considering the following recommendation scenario: for a target user $u$ under a contextual situation $cs$. The features of an item can have positive, negative, or neutral impacts on the prediction. In this regard, the features can be viewed as arguments, along with another argument stating ``the item can be recommended to the target user'', $rec^i$. Such a scenario can be seamlessly integrated into a TAF \cite{gabbay2016logical}. For a user-item interaction $(u,i)$ under $cs$, the TAF tailored to this interaction is a quadruple: $<\mathcal{A}, \mathcal{R}^{-}, \mathcal{R}^{+}, \mathcal{R}^{0}>$, where $\mathcal{A}$ contains $at_i$ and $rec^i$. To illustrate this idea, consider the following example where $\mathcal{R}^{+}(at,rec^i)$ denotes that the feature $at$ has a positive effect on the recommendation of item $i$. Figure~\ref{fig:explanation} visually demonstrates this idea. Hence, the objectives are as follows: (i) to predict the rating assigned by a target user $u$ to a particular item $i$ in a given contextual situation $cs$; (ii) to determine the contribution of each item feature to a prediction under this contextual situation $cs$; (iii) to assess the polarity of each argument (feature) in the TAF; (iv) to design the strength function of arguments that comply with the conditions of \emph{weak balance} and \emph{weak monotonicity} as defined above.

\section{Our framework: CA-FATA}\label{sec:our_proposition}

The core ideas of CA-FATA are as follows: (i) Users' ratings towards items depend on the features of items, which is also a fundamental concept of content-based RSs; (ii) The importance of feature types varies across users. For instance, when choosing a movie, some users may have specific preferences for certain actors while others may prefer movies of particular directors. In the former case, users pay more attention to the feature type \textbf{\emph{stars\_in}}, while in the latter case, they accord more importance to the feature type \textbf{\emph{directs}}; (iii) CA-FATA further extends this idea by noting that users' preferences also vary across contexts \cite{adomavicius2011context}. In the context of movie recommendations, some users are more likely to be influenced by \textbf{\emph{companion}} (e.g., with a lover or children) while others may be more likely to be influenced by \textbf{\emph{location}} (e.g., home or public place); (iv) We map item-features graphs into TAFs, where features of items are regarded as arguments that may attack, support, or neutralize the recommendation of the item; (v) The dialectical relations are learned in a data-driven way and supervised by users' past interactions with items (aka. ratings towards items).


Figure~\ref{fig:framework} depicts the structure of CA-FATA, which consists of three steps: (i) Computing the representation of target users under the target contextual situation to ensure that users' preferences are adapted to contexts and that the dialectical relations of arguments are also context-aware; (ii) Computing users' ratings towards features of items under the given contextual situation, which are then used to determine the dialectical relations; (iii) Aggregating the ratings obtained in the previous step to generate users' ratings towards items.

\textbf{Computing user representation (Step 1):} As illustrated in Figure~\ref{fig:framework}, the ultimate representation of a user is determined by the user's contextual situation. In this step, our goal is to compute the representation of users that is adapted to the target contextual situation. To achieve this, we begin by computing the score of the contextual factor $cf$ for a user $u$:

\begin{equation}\label{equa:score_context}
    \beta_u^{cf} = g(\textbf{u},\textbf{cf})
\end{equation}

where $g$ is the inner product\footnote[1]{Note that other functions may be adopted, but for simplicity, we adopt the inner product.}. After calculating the score $\beta_u^{cf}$ of all contextual factors to this user, {we normalize the score using Equation~\ref{equa:normalize} as described in \cite{velivckovic2017graph}} to obtain the importance of each contextual factor. The importance computed here is similar to the relevance weight in \cite{budan2020proximity,budan2020similarity}. However, unlike in these two works, where the relevance weight of the context is set empirically, in our work, the importance of the context is learned in a data-driven way. Intuitively, $\pi_u^{cf}$ characterizes the extent to which user $u$ wants to take contextual factor $cf$ into account.

\begin{equation}\label{equa:normalize}
   \pi_u^{cf} = \frac{exp(LeakyReLU(\beta_u^{cf}))}{\sum_{cf \in C}exp(LeakyReLU(\beta_u^{cf}))}
\end{equation}

\begin{figure}[t]
\centering
\includegraphics[scale=0.20]{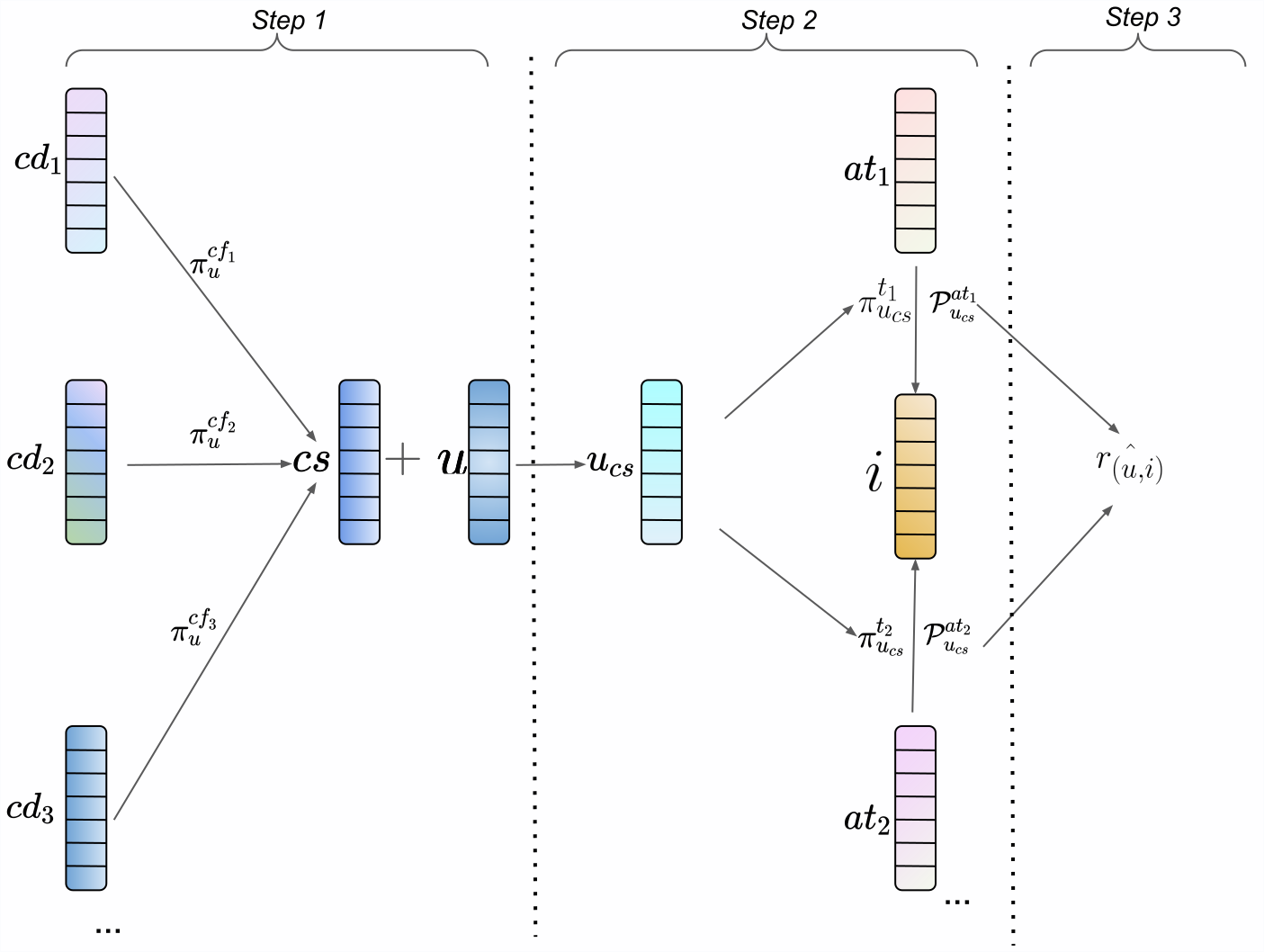}
\caption{Illustration of the framework of CA-FATA.}
\label{fig:framework}
\vspace{-0,4cm}
\end{figure}

In sequence, we compute the representation of the contextual situation $cs$ by summing up all the vectors representing contextual conditions multiplied by $\pi_u^{cf}$.

\begin{equation}\label{equa:compute_cs}
         \textbf{cs} = \sum\limits_{cd \in cs} \pi _u^{cf}\textbf{cd}
     \end{equation}
where $\textbf{cs}$ is the vector that denotes contextual situation $cs$. The next step is to aggregate the representation of contextual situation $cs$ with the representation of user $u$ to obtain a specific representation of user $u$ under the contextual situation $cs$. To avoid having an excessive number of parameters\footnote[2]{{Note that other aggregation methods such as concatenation are also possible but more parameters are induced. We leave this exploration for future work.}}, we sum $\textbf{u}$ and $\textbf{cs}$.  By aggregating information from a contextual situation $cs$ and a user $u$, each user $u$  gets a specific representation $\bm{u_{cs}}$ under a contextual situation $cs$. 

\begin{equation}\label{equa:user_final}
         \bm{u_{cs}} =\textbf{u} + \textbf{cs}
     \end{equation}

\textbf{Computing users' ratings towards features (Step 2):} {The feature types in this paper are similar to the relations in knowledge graphs}, which are directed graphs consist of \emph{entity-relation-entity} triplets \cite{hogan2021knowledge}. For instance, the triplet $(Harry Potter, hasDirector, Mike Newell)$ indicates that the movie \emph{Harry Potter} is directed by \emph{Mike Newell}. Here, $hasDirector$ is a relation in the knowledge graph that pertains to movies, and in this paper, it corresponds to the feature type $director$. We quantify the score between a feature type $t$ and a user $u$ as proposed in \cite{wang2019knowledge}:

\begin{equation} \label{equa:importance}
    \beta_{u_{cs}}^{t} = g(\bm{u_{cs}},\textbf{t})
\end{equation}
where $\bm{u_{cs}}$ is computed in the previous step, $\textbf{t}$ denotes the vector that represents feature type $t$. After calculating the score $\beta_{u_{cs}}^{t}$ of all feature importance to this user, {we normalize the score using Equation~\ref{equa:normalize_r} as described in \cite{velivckovic2017graph} to obtain the importance of each feature type}.

\begin{equation}\label{equa:normalize_r}
   \pi_{u_{cs}}^t = \frac{exp(LeakyReLU(\beta_{u_{cs}}^{t}))}{\sum_{t \in t_i}exp(LeakyReLU(\beta_{u_{cs}}^{t}))}
\end{equation}

To compute users' ratings towards features, we adopt the inner product again: $\mathcal{P}_{u_{cs}}^{at} = g(\bm{u_{cs}}, \textbf{at})$. According to Equations~\ref{equa:normalize}, \ref{equa:compute_cs}, and \ref{equa:user_final}, the representation of a user under one context differs from that under another context. As a result, the representation of user $u_{cs}$ is specific to each context, and the importance of feature type and user's rating towards features is also context-aware. 

\textbf{Aggregating ratings towards features (Step 3):} After calculating the importance of each feature type and users' ratings towards each feature, the next step is to compute the rating of user $u$ towards item $i$ by computing the contribution of each feature type $t$ using the following equation:

\begin{equation} \label{equa:contribution}
contr_t = \frac{\sum_{at \in at_i^t }\mathcal{P}_{u_{cs}}^{at}}{|at_i^t|}
\end{equation}
where $at_i^t$ denotes the features belong to type $t$ of item $i$. Finally, $u$'s rating towards $i$ under $cs$ is: 

\begin{equation} \label{equa:final}
\hat{r}_{(u,i)} = \sum_{t\in t_i}\pi_{u_{cs}}^t * contr_t
\end{equation} 
where $t_i$ denotes all the feature types of item $i$. It should be noted that the actual value of the user $u$'s rating for item $i$ is a real number between $-1$ and $1$, as defined in previous works such as \cite{rago2018argumentation,rago2021argumentative}. It is noteworthy that Equation~\ref{equa:final} is remarkably similar to Equation~\ref{equa:gams}, indicating that our model belongs to the family of generalized additive models. This similarity allows for easy identification of the contribution of each feature, fulfilling the first two goals: \textbf{{computing users' ratings towards items and determining the contribution of each feature}}.





In this section, we have presented how CA-FATA computes ratings. In the following section, we demonstrate how this process can be seamlessly integrated into an argumentation procedure.


\section{Argumentation scaffolding}\label{sec:argumentation}

\subsection{Argumentation setting}
CA-FATA first computes users' ratings towards features and then aggregates these ratings to predict ratings towards items. As a result, the features of items can be considered as arguments, and users' ratings towards features can be seen as the strength of these arguments. If CA-FATA predicts that a user's rating towards a feature of an item is high (under a target contextual situation $cs$), then this feature can be viewed as an argument that supports the recommendation of the item under $cs$. Conversely, if the predicted rating is low, the feature can be considered as an argument against the recommendation of the item under $cs$. In cases where features do not attack or support, a neutralizing relation is added, represented by $\mathcal{P}_{u_{cs}}^{at} = 0$. Therefore, we set $\sigma(at) = \mathcal{P}_{u_{cs}}^{at}$ and adopt TAF that contains support, attack, and neutralizing. Moreover, we set $\sigma(rec^i) =\hat{r}{(u,i)}$. When $\hat{r}{(u,i)} > 0$, the argument $rec^i$ is stronger if $\hat{r}_{(u,i)}$ is larger \footnote[3]{Note that if $\hat{r}_{(u,i)} < 0$ then the smaller $\hat{r}_{(u,i)}$ is, the stronger argument ``not recommending item $i$'' is}. 

Recall that the true rating $r_{(u,i)}$ is a real number between -1 and 1, then the co-domain of $\mathcal{P}_{u_{cs}}^{at}$ is also expected to be between -1 and 1. Therefore, when $\mathcal{P}_{u_{cs}}^{at} > 0$, then $\sigma(at) > 0$, indicating that $at$ is an argument that supports $rec^i$ \footnote[4]{Semantically, users $u$ prefers feature $at$}; when $\mathcal{P}_{u_{cs}}^{at} = 0$, then $\sigma(at) = 0$, indicating that $at$ is an argument that neutralizes $rec^i$; when $\mathcal{P}_{u_{cs}}^{at} < 0$, then $\sigma(at) < 0$, indicating that $at$ is an argument that attacks $rec^i$. Therefore, the TAF corresponds to a user-item interaction $(u,i)$ under $cs$ can be defined as follows: 

\begin{definition}\label{def:taf}
The TAF corresponding to $(u,i)$ under $cs$ is a 4-tuple: $<\mathcal{A}, \mathcal{R}^{-}, \mathcal{R}^{+}, \mathcal{R}^{0}>$ such that:
\begin{center}
$\bullet$ $\mathcal{R}^{-} = \{ (at,rec^i)|{\mathcal{P}_{u_{cs}}^{at}} < 0\}$;
\\
$\bullet$ $\mathcal{R}^{+} = \{(at,rec^i)|{\mathcal{P}_{u_{cs}}^{at}} > 0\}$;
\\
$\bullet$ $\mathcal{R}^{0} = \{(at,rec^i)|{\mathcal{P}_{u_{cs}}^{at}} = 0\}$.
\end{center}
\end{definition} 

According to the definition, ${\mathcal{P}_{u_{cs}}^{at}}$ determines the polarity of arguments: if ${\mathcal{P}_{u_{cs}}^{at}}$ is positive then the argument (feature) supports the recommendation of item $i$ to user $u$; if ${\mathcal{P}_{u_{cs}}^{at}}$ is negative then the argument (feature) attacks the recommendation of item $i$ to user $u$; if ${\mathcal{P}_{u_{cs}}^{at}}$ is 0 then the argument neutralizes the recommendation. Therefore, the third goal: \textbf{{determining the polarity of each argument (feature)}}, is fulfilled.

\subsection{Proofs}

We will now show that by setting $\sigma(at) = \mathcal{P}_{u_{cs}}^{at}$ and $\sigma(rec^i) =\hat{r}{(u,i)}$,  TAF corresponding to $(u,i)$ under $cs$ satisfies \emph{weak balance} and \emph{weak monotonicity}. Recall that \emph{weak balance} states that attacks (or supports) can be characterized as links between affecters and affectees in a way such that if one affecter is isolated as the only argument that affects the affectee, then the former reduces (increases, resp.) the latter's predicted rating with respect to the neutral point.

\begin{proposition}\label{proposition:balance}
Given the TAF corresponding to $(u,i)$ under $cs$ satisfies \emph{weak balance}.
\end{proposition}

\begin{proof}
By inspecting Equations~\ref{equa:contribution}, \ref{equa:final}, it can be observed that since users' ratings are transformed into $[-1,1]$, the co-domain of $\sigma$ is also $[-1,1]$. Case (i): if $\mathcal{R}^+(rec^i) = \{at\}$, $\mathcal{R}^-(rec^i) = \emptyset$ and $\mathcal{R}^0(rec^i) = \emptyset$ then $\sigma(at) > 0$, therefore, $\mathcal{P}_{u_{cs}}^{at} > 0$, according to Equations~\ref{equa:contribution} and ~\ref{equa:final}, $\hat{r}{(u,i)} > 0 $, indicating that $\sigma(rec^i) > 0$. Case (ii): if $\mathcal{R}^-(rec^i) = \{at\}$, $\mathcal{R}^+(rec^i) = \emptyset$ and $\mathcal{R}^0(rec^i) = \emptyset$ then $\sigma(at) < 0$, therefore, $\mathcal{P}_{u_{cs}}^{at} < 0$, according to Equations~\ref{equa:contribution} and ~\ref{equa:final}, $\hat{r}{(u,i)} < 0 $, indicating that $\sigma(rec^i) < 0$. Case (iii): if $\mathcal{R}^0(rec^i) = \{at\}$, $\mathcal{R}^+(rec^i) = \emptyset$ and $\mathcal{R}^-(rec^i) = \emptyset$ then $\sigma(at) = 0$, therefore, $\mathcal{P}_{u_{cs}}^{at} = 0$, according to Equations~\ref{equa:contribution} and ~\ref{equa:final}, $\hat{r}{(u,i)} = 0 $, indicating that $\sigma(rec^i) = 0$. 
\end{proof}

Similar to  \emph{weak balance}, \emph{weak monotonicity} characterizes attacks, supports, and neutralizes as links between arguments such that if the strength of one affecter is muted then the strength of its affectees increases, decreases, and remains unchanged, respectively. Therefore, \emph{weak monotonicity} highlights the positive/negative/neutral effect between arguments. In this way, \emph{weak monotonicity} reveals the positive, negative, or neutral effect of an argument.

\begin{proposition}\label{proposition:monotonicity}
Given the TAF corresponding to $(u,i)$ under $cs$, $\sigma(at) = \mathcal{P}_{u_{cs}}^{at}$ and $\sigma(i)$ satisfy the \emph{weak monotonicity}.
\end{proposition}

\begin{proof}
\emph{Weak monotonicity} is formulated for two TAFs: from $<\mathcal{A}, \mathcal{R}^{-}, \mathcal{R}^{+}, \mathcal{R}^{0}>$ to $<\mathcal{A}^{\prime}, \mathcal{R}^{{-}^{\prime}}, \mathcal{R}^{{+}^{\prime}}, \mathcal{R}^{{0}^{\prime}}>$, after modifying certain arguments (e.g. muting certain features). Case (i): if $at \in \mathcal{R}^-(rec^i)$, then according to Definition~\ref{def:taf}, ${\mathcal{P}_{u_{cs}}^{at}} < 0$, when $at$ is muted then ${\mathcal{P}_{u_{cs}}^{at}}^{\prime} = 0$. According to Equations~\ref{equa:contribution} and ~\ref{equa:final}, ${\hat{r}_{(u,i)}} ^{\prime} > \hat{r}_{(u,i)} $, indicating that ${\sigma(rec^i)}^{\prime} > \sigma(rec^i)$. Case (ii): if $at \in \mathcal{R}^+(rec^i)$, then according to Definition~\ref{def:taf}, ${\mathcal{P}_{u_{cs}}^{at}} > 0$, when $at$ is muted then ${\mathcal{P}_{u_{cs}}^{at}}^{\prime} = 0$. According to Equations~\ref{equa:contribution} and ~\ref{equa:final}, ${\hat{r}_{(u,i)}} ^{\prime} < \hat{r}_{(u,i)} $, indicating that ${\sigma(rec^i)}^{\prime} < \sigma(rec^i)$. Case (iii): if $at \in \mathcal{R}^0(rec^i)$, then according to Definition~\ref{def:taf}, ${\mathcal{P}_{u_{cs}}^{at}} = 0$, when $at$ is muted then ${\mathcal{P}_{u_{cs}}^{at}}^{\prime} = 0$. According to Equations~\ref{equa:contribution} and ~\ref{equa:final}, ${\hat{r}_{(u,i)}} ^{\prime} = \hat{r}_{(u,i)} $, indicating that ${\sigma(rec^i)}^{\prime} = \sigma(rec^i)$.
\end{proof}

Based on Proposition~\ref{proposition:monotonicity}, the following corollary holds.
\begin{corallary}
If the predicted rating of on $at \in at_i$ is increased: ${\mathcal{P}_{u_{cs}}^{at}}^{\prime} > \mathcal{P}_{u_{cs}}^{at}$, then the predicted rating of item $i$ also increases: ${\hat{r}_{(u,i)}} ^{\prime} > \hat{r}_{(u,i)}$. Accordingly, if the predicted rating one $at \in at_i$ is decreased: ${\mathcal{P}_{u_{cs}}^{at}}^{\prime} < \mathcal{P}_{u_{cs}}^{at}$, then the predicted rating of item $i$ also decreases: ${\hat{r}_{(u,i)}} ^{\prime} < \hat{r}_{(u,i)}$.
\end{corallary}

\begin{proof}
    By inspecting Equations~~\ref{equa:contribution}, \ref{equa:final}, Proposition~\ref{proposition:monotonicity}.
\end{proof}

Propositions~\ref{proposition:balance} and ~\ref{proposition:monotonicity} ensure that the forth objective: \textbf{{designing strength function that satisfies\emph{weak balance} and \emph{weak monotonicity}}}, is fulfilled. The TAF corresponding to the CA-FATA shows how each feature influences the final rating prediction (recommendation) in an argumentative manner: each feature may support, attack, or neutralize the recommendation of the item, and the user's rating towards each feature determines its polarity. That being said, CA-TAFA is decomposable \citep{lipton2018mythos}, as each computation has explicit meanings that can also be articulated as intelligible \citep{lou2012intelligible}.

As a running example, Figure~\ref{fig:explanation} presents the TAF for a user-item interaction under the contextual situation $cs$ (extracted from Figure~\ref{fig:framework}). In this TAF, each feature of the item represents an argument. The user's rating towards each feature determines the strength and polarity of the argument, thereby reflecting the user's preference. The strength of argument 1 is $+0.52$, indicating that the user likes feature 1 (e.g., a movie director or actor), and this feature supports the recommendation of the item to the user. The strength of argument 2 is $-0.11$, indicating that the user does not like feature 2 and that this feature attacks the recommendation of the item to the user. Finally, the strength of argument 3 is $0$, indicating that this feature does not influence the user's rating. Note that, according to the three steps in Section~\ref{sec:our_proposition}, the prediction score could differ under different contexts, even for the same user and item. Therefore, the corresponding TAFs would also differ.

\subsection{Explanation scenarios}
\begin{table*}[t]
\caption{Three explanation templates for user-item interaction under contextual situation $cs = (cd_1, cd_2, cd_3, \dots)$, $SR$ denotes ``strong recommendation'', $WR$ denotes ``weak recommendation'', $NR$ ``not recommended''.}
\centering
\begin{tabular}{|c|c|l|}
\hline
Scenario & Content & Example \\ \hline
\multirow{3}{*}{$SR$} & $at_1 = \mathop{\arg\max}\limits_{at \in at_i}\mathcal{P}_{u_{cs}}^{at}$ & When $cd$, we recommend you this item \\
 & $at_2 = \mathop{\arg\max}\limits_{at \in at_i \backslash at_1}\mathcal{P}_{u_{cs}}^{at}$  &because you like $at_1$ and $at_2$. \\
&$cd = \mathop{\arg\max}\limits_{cd \in cs}\pi_u^{cf}$&
\\ \hline
 
\multirow{3}{*}{$WR$} & $at_1 = \mathop{\arg\max}\limits_{at \in at_i}\mathcal{P}_{u_{cs}}^{at}$ & When $cd$, we recommend you this item \\ 
& $at_2 = \mathop{\arg\min}\limits_{at \in at_i \backslash at_1}\mathcal{P}_{u_{cs}}^{at}$  & because you like $at_1$ although you  \\
&$cd = \mathop{\arg\max}\limits_{cd \in cs}\pi_u^{cf}$& dislike $at_2$.
\\ \hline

\multirow{3}{*}{$NR$} & $at_1 = \mathop{\arg\min}\limits_{at \in at_i}\mathcal{P}_{u_{cs}}^{at}$ & When $cd$, we do not recommend you \\& $at_2 = \mathop{\arg\min}\limits_{at \in at_i \backslash at_1}\mathcal{P}_{u_{cs}}^{at}$ &this item because you dislike $at_1$ and $at_2$. \\ &$cd = \mathop{\arg\max}\limits_{cd \in cs}\pi_u^{cf}$&\\ \hline
\end{tabular}
\label{tab:scenario}
\vspace{-0.4cm}
\end{table*}

After conducting the above analyses, we propose three explanation templates in Table~\ref{tab:scenario}, similar to the three explanation types in \cite{rago2021argumentative}, but with the inclusion of users' contexts. In each scenario, we select the most influential contextual condition (as determined by Equation~\ref{equa:score_context}). For ``strong recommendation'', we propose selecting the two strongest arguments (aka. features) that support the recommendation of the item. In ``weak recommendation'', we propose selecting the strongest argument that supports the recommendation of the item and the strongest one that attacks the recommendation of the item. In ``not recommended'', we propose selecting the two strongest arguments (aka. features) that attack the recommendation of the item. Each template includes contextual information along with the corresponding arguments that either support or attack the recommendation of the item. {In summary, CA-FATA is a versatile model that can be used to explain both the reasons for recommended items as well as the reasons why some items should not be recommended. Additionally, users have the flexibility to define their own templates according to their specific needs.}

\section{Experiments}\label{sec:experiments}

In this section, we conduct experiments on two real-world datasets to address the following research questions: \textbf{RQ1}, can CA-FATA achieve competitive performance compared to baseline methods? What are the advantages of CA-FATA compared to baseline methods? \textbf{RQ2}, how does context influence the performance of CA-FATA? \textbf{RQ3}, how does the importance of feature type affect the performance of CA-FATA?
\subsection{Datasets and experiment setting}
\subsubsection{Datasets}

We have conducted experiments on the following real-world datasets:

\textbf{Frapp\'e:} This dataset is collected by \cite{baltrunas2015frappe}. This dataset originated from Frapp\'e, a context-aware app recommender. There are 96 303 logs of usage from 957 users under different contextual situations, 4 082 apps are included in the dataset. Following \cite{unger2020context}, we apply log transformation to the number of interactions. As a result, the number of interactions is scaled to $0-4.46$. Each contextual situation is composed of 7 contextual conditions: ``daytime'', ``weekday'', ``isweekend'', ``homework'', ``weather'', ``country'', ``city'' and each of them corresponds to a contextual factor. The features of each app include: ``category'', ``number of downloads'', ``language``, ``price'', ``average rating (given by other users)''.

\begin{table}[t]
\centering
\caption{Basic statistics of the two datasets}
\label{tab:statistics}
\begin{tabular}{@{}lcc@{}}
\hline
Dataset  & Frapp\'e & Yelp \\ \hline
\#users  & 585 & 9 976 \\
\#items & 3 923 & 52 298 \\
\#interactions & 94 716  & 904 648 \\
sparsity (users $\times$ items) & 94.47\%& 99.84\%\\
\#Contextual factors & 7 & 8 \\
\#Feature types & 5& 3  \\
Scale & 0-4.46 & 1-5 
\\ \hline
\end{tabular}
\vspace{-0,4cm}
\end{table}

\textbf{Yelp\footnote[5]{\url{https://www.yelp.com/dataset}}:} This dataset contains users' reviews on bars and restaurants in metropolitan areas in the USA and Canada. Consistent with previous studies by \cite{zhou2020s3,geng2022recommendation}, we use the records between January 1, 2019 to December 31, 2019, which contains 904 648 observations. As contextual factors, we have derived ``month'', ``day of week'', ``isweekend'', ``time of the day'' from the timestamp of the records; ``alone\_or\_companion'' is extracted from the reviews of users, ``state'' and ``city'' are provided. The features available include: ``stars'' (average ratings given by other users), ``review count '' ( the number of reviews received), and ``item type''.

For the two datasets, we have adopted the 10-core setting, following \cite{wang2019neural}, to ensure data quality. This means that only users with at least 10 interactions are kept. Detailed statistics on the pre-processed datasets are presented in Table~\ref{tab:statistics}. 

\subsubsection{Baselines and evaluation setting}

We compare the following baselines: (i) \textbf{MF} \cite{koren2009matrix}: This classic collaborative filtering method only considers user-item pairs and computes the inner product of the vectors representing users and items to make predictions, without taking into account additional information such as users' contexts and item features. (ii) \textbf{CAMF-C} \cite{baltrunas2011matrix}: An extension of \textbf{MF} that incorporates the global influence of contexts on ratings. (iii) \textbf{FM} \cite{rendle2010factorization}: A strong baseline that models the second-order interactions between all information related to user-item interactions, including users' characteristics, users' context, and item features. (iv) \textbf{NeuMF} \cite{he2017neural1}: A method that combines matrix factorization and MLP (Multi-Layer Perceptron) to model the latent features of users and items. (v) \textbf{ECAM-NeuMF} \cite{unger2020context}: An extension of \textbf{NeuMF} that integrates contextual information. Note that the authors in \cite{unger2020context} do not release the implementation detail, for the NeuMF part, we empirically set the MLP factor size as 8, the sizes of the hidden layer as $(16,8,4)$, the GMF (Generalized Matrix Factorization) factor size as 16. This setting also applies to pure NeuMF. (vii) \textbf{A-I} \cite{rago2018argumentation, rago2021argumentative,rago2020argumentation}: An argumentation-based framework that computes users' ratings towards features, which are then aggregated to obtain the ratings towards items. Following \cite{rago2018argumentation}\footnote[6]{For detail please refer to \url{https://github.com/CLArg-group/KR2020-Aspect-Item-Recommender-System}.}, we set the ``collaborative factor'' as 0.8, 20 most similar users are selected, and  all the feature importance is set as 0.1. 

In these two datasets, users give explicit ratings towards items, therefore the squared loss is utilized to optimize parameters of CA-FATA:
\begin{equation}
    L = \sum_{(u,i,cs) \in \mathcal{T}}{(\hat{r}_{(u,i,cs)} - r_{(u,i,cs)})}^2 + \lambda {\left \| \Theta \right\|}_2^2
\end{equation} 
where $\mathcal{T}$ is the training set, $\hat{r}_{(u,i,cs)}$ is the predicted rating and $r_{(u,i,cs)}$ denotes the actual rating, $\lambda$ denotes the regularization parameter to reduce over-fitting, $\Theta$ denotes the parameters of CA-FATA. We implement CA-FATA using Pytorch\footnote[7]{Access to source code is provided in \url{https://anonymous.4open.science/r/CA_FATA-F0B3}} and we optimize the parameters using \emph{mini-batch Adam}. The testing platform is Tesla P100-PCIE, 16GB memory in CPU. The hyper-parameters are tuned through a grid search: the learning rate is tuned on $[10^{-5}, 10^{-4}, 10^{-3}, 10^{-2}, 10^{-1}]$; the batch size is tuned on $[128,256,512,1024,2048,4096]$; regularization parameter is tuned on range $[5 * 10^{-5}, 10^{-4}, 5 * 10^{-3}, 10^{-3}, 10^{-2}]$. The embedding size is tuned on $[8,16,32,64,128,256]$. Root Mean Squared Error (RMSE), and Mean Absolute Error (MAE) are selected as the primary evaluation metrics. We follow the convention established in \cite{unger2020context,li2020interpretable,chen2022tinykg} by splitting the datasets into a training set, a test set, and a validation set, with a ratio of $8:1:1$. 



\subsection{Rating prediction (RQ1)}
\begin{table*}[t]
\centering
\caption{Comparison between CA-FATA and baselines on RMSE and MAE, the second best are underlined. FATA is basically a variant of CA-FATA, the difference between CA-FATA and FATA is that FATA does not consider users' contexts and is actually the $A^3R$ \cite{zhong20223} model. The version with ``AVG'' means that the importance of each feature type is set the same for all users.}
\label{tab:results}
\begin{tabular}{ccllll}
\toprule
\multicolumn{2}{c}{\multirow{2}{*}{Model}} & \multicolumn{2}{c}{Yelp} & \multicolumn{2}{c}{Frapp\'e} \\
\cmidrule(l){3-6} 
\multicolumn{2}{c}{}  & RMSE & MAE & RMSE  & MAE \\ \toprule
\multirow{2}{*}{Context-free}  & MF  & 1.1809  & 0.9446 & 0.8761 & 0.6470 \\

 &NeuMF &1.1710&0.8815&0.6841 & 0.5207
 \\ \toprule
\multirow{5}{*}{Context-aware} & FM  & 1.1703 & 0.9412& 0.7067 & 0.5796\\
& CAMF-C & 1.1693  & 0.9241 & 0.7283 & 0.5727 \\
& LCM & 1.1687  & 0.9294 & 0.6952& 0.5396 \\
 & ECAM-NeuMF  & 1.1098  & \underline{0.8636} & 0.5599 &   0.4273\\
 \toprule
Argumentation-based  & A-I & 1.3978  & 1.1205& 1.1711 & 0.9848 \\\toprule
\multirow{4}{*}{Our propositions} & FATA & 1.1434 & 0.9059 & 0.6950 & 0.5439 \\
&AVG-FATA &1.1611&0.9314&0.6970&0.5461 \\
& \textbf{CA-FATA}  &\textbf{1.1033}&\textbf{0.8519} & \textbf{0.5154} & \textbf{0.3910} \\
&AVG-CA-FATA&\underline{1.1035}&0.8637&\underline{0.5254}&\underline{0.4025}\\

 \bottomrule
\end{tabular}
\vspace{-0,4cm}
\end{table*}

Table~\ref{tab:results} presents the results of the rating prediction experiments. We observe that CA-FATA performs better than all baselines on both the Yelp dataset and the Frapp\'e dataset, indicating its superiority in handling complex contextual information. The following are some specific observations:

\textbf{On the dataset Yelp:} (i) CA-FATA outperforms MF in terms of recommendation accuracy because it considers users' contexts and models their preferences towards features. In contrast, MF only explores user-item interactions. NeuMF, which combines matrix factorization and MLP to capture the latent features of users and items, performs worse than CA-FATA. Furthermore, every computation step in CA-FATA has explicit semantics, which contributes to better explainability of recommendations. CA-FATA further considers the influence of contexts and item features, which will be further discussed in Section~\ref{sec:Abalation}. (ii) Among the context-aware models, CA-FATA outperforms CAMF-C, FM, LCM, ECAM-NeuMF. Moreover, CA-FATA can explain recommendations in an argumentative way, as discussed in Section~\ref{sec:argumentation}. (iii) CA-FATA significantly improves prediction accuracy compared to A-I, with an RMSE improvement of $21.06\%$ and an MAE improvement of $23.97\%$. This result can be interpreted in two ways: (1) CA-FATA learns the importance of feature types to users under different contexts, while A-I assigns the same importance to feature types for all users and does not consider users' contexts; (2) Selecting the number of the $k$ most similar users in A-I is not straightforward because a small $k$ may ignore important users, while a large $k$ may introduce noise from less similar users. 


\textbf{On the dataset Frapp\'e:} (i) CA-FATA performs exceptionally well on this dataset, outperforming all baselines. Compared to context-free methods, CA-FATA improves upon MF by $40.86\%$ and $39.56\%$ on RMSE and MAE, respectively. Compared to NeuMF, CA-FATA achieves reductions in RMSE and MAE by $24.28\%$ and $24.03\%$, respectively. (ii) Among the context-aware baselines, CA-FATA outperforms FM, CAMF-C, LCM, and ECAM-NeuMF, demonstrating its ability to model users' preferences under different contexts. Another advantage of CA-FATA is the ability to provide argumentative explanations, which is not possible for these baselines. (iii) Compared to A-I, on Yelp, CA-FATA achieves a significant reduction in RMSE and MAE by $55.99\%$ and $60.29\%$, respectively. (iv) A horizontal comparison of Frappé and Yelp datasets shows that CA-FATA performs better on Frappé than on Yelp. We attribute this difference to the sparsity of the dataset, as Yelp is still highly sparse even after applying the 10-core setting, with a sparsity of $99.84\%$ (see Table~\ref{tab:statistics}), while Frappé has a sparsity of $94.47\%$.

\begin{table*}[t]
\centering
\caption{Comparison between CA-FATA and baselines on precision, recall and f1 score, the second best are underlined.}
\label{tab:results_hr}
\begin{tabular}{ccccccc}
\hline
\multirow{2}{*}{Model} & \multicolumn{3}{c}{Yelp} & \multicolumn{3}{c}{Frapp\'e} \\ 
\cmidrule(r){2-4} \cmidrule(r){5-7}

  &Precision& Recall & f1 & Precision& Recall & f1 \\ \hline
     MF  & 0.8346& 0.8907 &0.8617 & 0.6889 &0.7713&0.7278\\

 NeuMF & 0.8962 &0.9522&0.9234  &0.6882 &0.8151 &0.7463
 \\ \hline
 FM  &0.8654  &0.9028  &0.8837 &0.6996 &0.7896&0.7337\\
CAMF-C &0.8701  & 0.9176 &0.8932& 0.6913 &0.7804&0.7331 \\
LCM & 0.8693 & 0.9207 &  0.8942&0.6894&0.7769& 0.7305 \\
 ECAM-NeuMF  &  0.8849&\textbf{0.9827}& \textbf{0.9313}  &0.7774  &0.8124& 0.7945\\ \hline
 A-I & 0.8249 & 0.8828 &0.8528&  0.6731&0.7658&0.7164\\\hline
 FATA & \underline{0.8982} & 0.9361 & 0.9167 &0.6752&0.8139&0.7381\\
AVG-FATA &  \textbf{0.8998}&0.9277  & 0.9135 &0.6748&0.8141&0.7380\\
 \textbf{CA-FATA}  & 0.8888 &\underline{0.9754}  & \underline{0.9303} & \textbf{0.8043}&\textbf{0.8336}&\textbf{0.8182}\\
AVG-CA-FATA& {0.8894} &  0.9748 &0.9302 &\underline{0.7982}&\underline{0.8308}&\underline{0.8142}\\
 \hline
\end{tabular}
\vspace{-0,4cm}
\end{table*}

Due to the subjective nature of ratings, different users might assign different ratings to the same item, even if they both enjoyed it. For instance, one user might give a rating of 5 stars while another might give it 4 or 4.5 stars. To account for such variations in subjective ratings, we follow the approach proposed by Rago et al. \citep{rago2021argumentative} and convert ratings to a binary scale. In our case, ratings greater than or equal to 3 in the Yelp dataset are considered positive, while ratings less than 3 are considered negative. We set the threshold to 0.9030, which is the average rating across all users. We report global precision, recall, and F1 score based on the binary scale, and present the results in Table~\ref{tab:results_hr}. The results in Table~\ref{tab:results_hr} confirm the trends observed in Table~\ref{tab:results}, showing that CA-FATA and its variants in most cases outperform the baselines. Furthermore, the advantage of using CA-FATA over other methods extends beyond its superior performance. CA-FATA provides transparency in the prediction process, allowing for the tracing of the feature attribution procedure. This added interpretability is a crucial advantage compared to other baseline models. 

To summarize, the advantages of CA-FATA are as follows: (i) it achieves competitive performance compared to both context-free and context-aware baselines. These baselines use factorization-based methods such as MF, FM, and CAMF-C, and some combine neural networks like NeuMF and ECAM-NeuMF, which makes them difficult to interpret. {On the other hand, CA-FATA provides explicit semantics for each computation and generates argumentative explanations (see Table~\ref{tab:scenario} for some examples)}; (ii) compared to the argumentation-based method A-I, CA-FATA significantly improves prediction accuracy and generates context-aware explanations.




\subsection{Abalation study (RQ2 and RQ3)} \label{sec:Abalation}

 \begin{figure*}[t]
\centering  
\subfigure[\footnotesize{Cluster of users according to importance of each contextual factor}]{
\label{Fig.sub.7}
\includegraphics[width=0.28\textwidth]{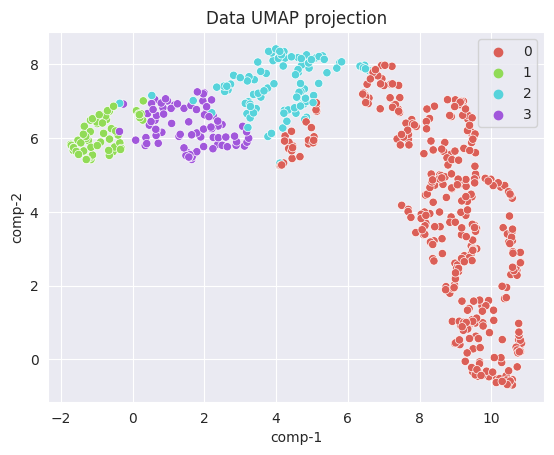}}
\subfigure[\footnotesize{Average importance of each contextual factor of users from cluster 0}]{
\label{Fig.sub.0}
\includegraphics[width=0.28\textwidth]{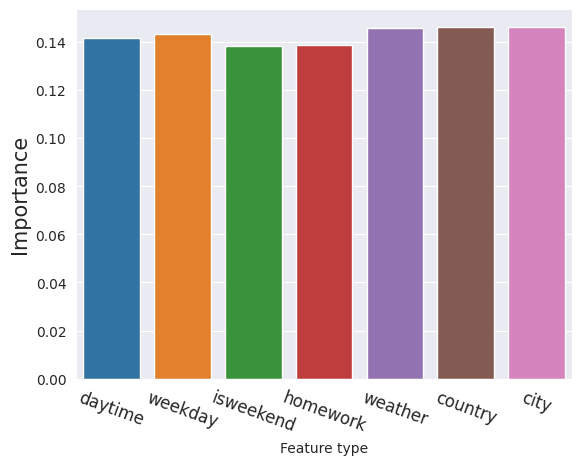}}
\subfigure[\footnotesize{Average importance of each contextual factor of users from cluster 1}]{
\label{Fig.sub.1}
\includegraphics[width=0.28\textwidth]{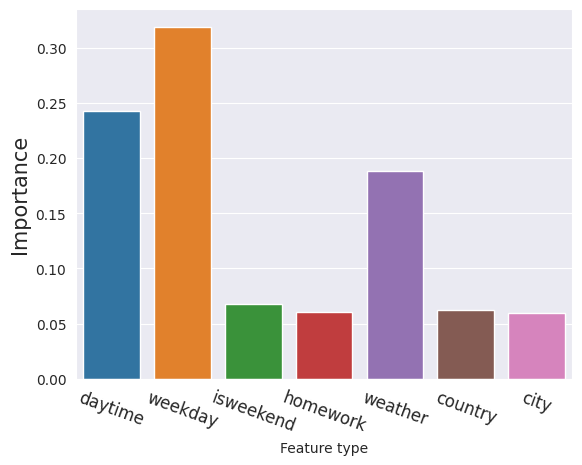}}
\subfigure[\footnotesize{Average importance of each contextual factor of users from cluster 2}]{
\label{Fig.sub.2}
\includegraphics[width=0.28\textwidth]{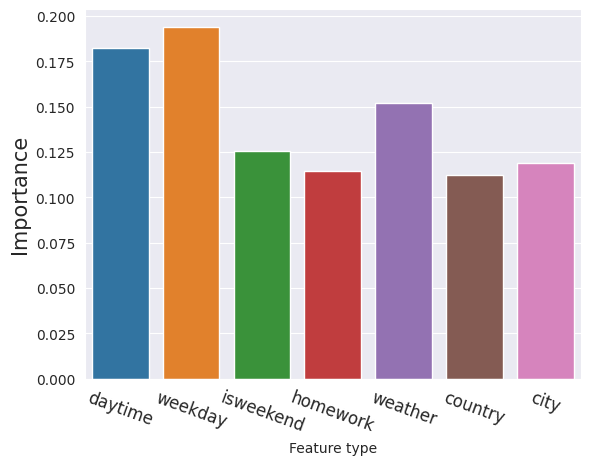}}
\subfigure[\footnotesize{Average importance of each contextual factor of users from cluster 3}]{
\label{Fig.sub.3}
\includegraphics[width=0.28\textwidth]{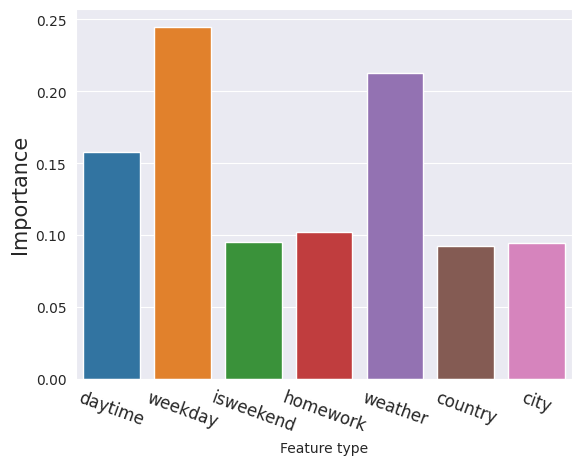}}
\caption{A case study on Frapp\'e that shows the clustering of users according to the contextual factor importance learned by $CA-FATA$. The histogram shows the average importance of each contextual factor in the cluster.}
\label{Fig.main}
\vspace{-0,4cm}
\end{figure*}

In order to investigate the impact of contextual factors on the performance of CA-FATA, we propose an alternative approach called FATA, which neglects user contexts and computes ratings using Equations~\ref{equa:importance},\ref{equa:normalize_r}, and~\ref{equa:final}, identical to the $A^3R$ model proposed in \cite{zhong20223}. Results presented in Tables~\ref{tab:results} and~\ref{tab:results_hr} (\textbf{refer to rows 11 and 13}) demonstrate that CA-FATA outperforms FATA, indicating that incorporating user contexts enables more nuanced modeling of user preferences and improves prediction accuracy. This conclusion is reinforced by the superior performance of CAMF-C over MF and ECAM-NeuMF over NeuMF.

To investigate the influence of feature type importance on our proposed model's performance, we introduce AVG-CA-FATA and AVG-FATA for CA-FATA and FATA, respectively (\textbf{refer to rows 12 and 14 in Tables~\ref{tab:results} and~\ref{tab:results_hr}}). In these models, the importance of each feature type is uniformly set for all users. For instance, in Frappé, where there are five feature types, the importance is set to 0.2 for all users, while in Yelp, where there are three feature types, the importance is set to 0.33. Results demonstrate that AVG-CA-FATA performs worse than CA-FATA, as does AVG-FATA when compared to FATA. Furthermore, comparisons between FATA, AVG-FATA, CA-FATA, and AVG-CA-FATA confirm the advantages of incorporating user contexts and modeling feature type importance across users.

To further visualize the impact of context, we represent each user by their contextual factor importance, computed using Equation~\ref{equa:normalize}. We use the Frappé dataset as an example, where a vector of seven dimensions represents each user: $(\pi_u^{daytime},\pi_u^{weekday},\pi_u^{isweekend},\pi_u^{homework},\pi_u^{weather},\pi_u^{country},\pi_u^{city})$. We apply K-means clustering for its simplicity and effectiveness \citep{velmurugan2010computational}, and find that four clusters best fit the dataset, as illustrated in Figure~\ref{Fig.sub.0}. We then use UMAP \citep{mcinnes2018umap} to visualize the clustering results. Note that other dimension reduction techniques could also be used, {but we choose UMAP because it can preserve the underlying information and general structure of the data.} The average importance of each contextual factor in the seven clusters is shown in Figures~\ref{Fig.sub.0}, ~\ref{Fig.sub.1}, ~\ref{Fig.sub.2}, ~\ref{Fig.sub.3}, revealing that users pay different levels of attention to contextual factors in the different clusters. Note that the same visualization applies to the Yelp data, due to limited space, we have omitted the visualization on the Yelp dataset.

{In this section, we conducted experiments on two real-world datasets to evaluate the performance of our proposed CA-FATA model. The results demonstrate that CA-FATA outperforms some neural network-based models in terms of recommendation accuracy. Additionally, CA-FATA has the advantage of providing explanations for recommendations through arguments, which is not possible with these neural network-based models. Furthermore, our model demonstrates significant improvement compared to existing argumentation-based models. The ablation study conducted on our proposed model highlights the benefits of leveraging users' contexts and modeling the importance of feature types.}

\section{Conclusions and perspectives}\label{sec:conclu}
{In light of the interpretability challenges associated with existing feature attribution methods,}  we present a novel feature attribution framework called \textbf{C}ontext-\textbf{A}ware \textbf{F}eature \textbf{A}ttribution \textbf{T}hrough \textbf{A}rgumentation (CA-FATA). CA-FATA is a feature attribution framework that treats features as arguments that can either support, attack, or neutralize a prediction using argumentation procedures. This approach provides explicit semantics to each step and allows for easy incorporation of user context to generate context-aware recommendations and explanations. The argumentation scaffolding in CA-FATA is designed to satisfy two important properties: \emph{weak balance} and \emph{weak monotonicity}, {which highlights how features influence a prediction}. These properties help identify important features and study how they influence the prediction task. We also introduce three explanation scenarios - strong recommendation, weak recommendation, and not recommended, {which can be used to explain why items have been recommended or not recommended.} Our experimental results show that CA-FATA outperforms several strong baselines regarding RMSE, MAE, precision, recall, and f1 score, highlighting its ability to provide both accuracy and explainability. In the future, we plan to explore the applicability of CA-FATA in other domains to verify its generalizability. {To compute the score of contextual factor and feature type, we adopted the inner product (in Equations~\ref{equa:importance} and ~\ref{equa:normalize_r}). We plan to explore other functions for this purpose.} Additionally, we intend to conduct user studies to compare the effectiveness of context-aware and argumentative explanations with other explanation methods.

\bibliographystyle{ACM-Reference-Format}
\bibliography{article}


\end{document}